\documentclass[letterpaper]{article}

\usepackage[utf8]{inputenc}
\usepackage{amsmath,amscd,amssymb,amsthm}
\usepackage{fullpage}
\usepackage{natbib}
\usepackage{hyperref}

\usepackage{microtype}
\usepackage{graphicx}
\usepackage{subfigure}
\usepackage{booktabs} 

\usepackage{amsmath,amscd,amssymb,amsthm}
\usepackage{verbatim}
\usepackage{thmtools}
\usepackage{thm-restate}
\usepackage{algorithm}
\usepackage{algorithmic}

\newcommand{\R}{\mathbb{R}}
\newcommand{\E}{\mathop{\mathbb{E}}}

\newcommand{\bw}{\vec{w}}
\newcommand{\bx}{\vec{x}}
\newcommand{\bz}{\vec{z}}
\newcommand{\bxi}{\mathbf{\xi}}

\newcommand{\bm}{\vec{m}}
\newcommand{\night}{{\texttt{NIGT}\ }}

\newcommand{\G}{\mathfrak{g}}

\newcommand{\D}{\mathcal{D}}

\usepackage{hyperref}

\title{Momentum Improves Normalized SGD}
\author{Ashok Cutkosky\\Google Research\\\texttt{ashok@cutkosky.com} \and Harsh Mehta\\
Google Research\\
\texttt{harshm@google.com}}
\date{}
\begin{document}

\maketitle

\begin{abstract}
We provide an improved analysis of normalized SGD showing that adding momentum provably removes the need for large batch sizes on non-convex objectives. Then, we consider the case of objectives with bounded second derivative and show that in this case a small tweak to the momentum formula allows normalized SGD with momentum to find an $\epsilon$-critical point in $O(1/\epsilon^{3.5})$ iterations, matching the best-known rates without accruing any logarithmic factors or dependence on dimension. We also provide an adaptive method that automatically improves convergence rates when the variance in the gradients is small. Finally, we show that our method is effective when employed on popular large scale tasks such as ResNet-50 and BERT pretraining, matching the performance of the disparate methods used to get state-of-the-art results on both tasks.
\end{abstract}

\section{Non-Convex Stochastic Optimization}
The rise of deep learning has focused research attention on the problem of solving optimization problems that are \emph{high-dimensional}, \emph{large-scale}, and \emph{non-convex}. Modern neural networks can have billions of parameters (high-dimensional) \citep{Raffel2019ExploringTL, shazeer2017outrageously}, are trained using datasets containing millions of examples (large scale) \citep{imagenet_cvpr09} on objective functions that are non-convex. Because of these considerations, stochastic gradient descent (SGD) has emerged as the de-facto method-of-choice for training deep models. SGD has several properties that make it attractive for this setting. First, it can operate in a streaming manner by running over the large dataset while only keeping a few examples in memory at any one time. Second, SGD has a \emph{dimension-free} convergence guarantee for finding critical points on non-convex objectives. This means that SGD's convergence rate is independent of the dimension of the problem, allowing it to scale more easily to massive neural networks. In this paper we will investigate a variant of SGD that incorporates \emph{gradient normalization} and \emph{momentum}, which are both commonly used empirical modifications to SGD that are poorly understood in the non-convex setting.

A common way to analyze of algorithms for training neural nets is through the lens of stochastic optimization. In this setting, we are interested in minimizing a function:
\begin{align*}
    F(\bw) = \E_{\bxi \sim \D}[f(\bw, \xi)]
\end{align*}
Where here $\bw\in \R^d$ and $D$ is a distribution over an arbitrary space $\Xi$. This general formulation allows for cleaner analysis, but for a more concrete intuition, $\bw$ may represent the weights of a neural network, $\bxi$ may represent an individual training example, and $\D$ is the distribution over training examples. We do not know the entire function $F$, but we are allowed to access $F$ through a \emph{stochastic gradient oracle}. That is, given any $\bw$, we may compute $\nabla f(\bw, \xi)$ for some $\xi$ drawn i.i.d. from $\D$. Under this formulation, the SGD algorithm employs the simple iterative update:
\begin{align*}
    \bw_{t+1} = \bw_t - \eta_t \nabla f(\bw_t, \xi_t)
\end{align*}
where $\xi_1,\dots,\xi_T$ are i.i.d. examples from the distribution $\D$, and $\eta_t\in \R$ is a a scalar called the \emph{learning rate}. Since our objective $F$ is non-convex, it may be computationally intractable to find a true minimizer. Instead, we will look for an $\epsilon$-critical point. That is, we would like our optimization algorithm to output a point $\bw$ such that
\begin{align*}
    \E[\|\nabla F(\bw)\|]\le \epsilon
\end{align*}
where $\|\cdot\|$ indicates the standard 2-norm. The expectation is taken over the randomness of the stochastic gradient oracle and any randomness inherent in the algorithm itself. In order to make the problem tractable, we make three structural assumptions. First, we assume that $F$ is $L$-smooth, which means that for all $\bw$ and $\bx$,
\begin{align}
    \|\nabla F(\bw) - \nabla F(\bx)\| \le L\|\bw - \bx\|\label{eqn:smoothness}\tag{A1}
\end{align}
Second, we assume that the objective $F$ is bounded from below, and without loss of generality (since our algorithms will only access gradients rather than function values), we may assume it is positive:
\begin{align}
    F(\bw)\ge 0 \label{eqn:positive}\tag{A2}
\end{align}
Finally, we assume that the stochastic gradient oracle has bounded variance:
\begin{align}
\E_{\bxi}\left[\|\nabla f(\bw, \bxi) - \nabla F(\bw)\|^2\right]\le \sigma^2 \label{eqn:variance}\tag{A3}
\end{align}
Under these assumptions, SGD with an optimally tuned learning rate ensures that after $T$ iterations, we can output a point $\bw$ such that \citep{ghadimi2013stochastic}:
\begin{align}
\E[\|\nabla F(\bw)\|]\le O\left(\frac{1}{\sqrt{T}} + \frac{\sqrt{\sigma}}{T^{1/4}}\right)\label{eqn:sgd}
\end{align}
Spurred by the empirical success of SGD, there has been a tremendous amount of work in designing modifications that improve upon its performance in various ways.
However, it has recently been shown that the $O(1/T^{1/4})$ rate is optimal in the worst-case \citep{arjevani2019lower}, and so improvements must either make more assumptions about the problem setting, or provide an algorithm that can somehow interpolate between a worst-case and non-worst-case problem.

One popular modification to SGD is the use of \emph{adaptive learning rates}, popularized by AdaGrad \citep{duchi10adagrad}. These learning rates enable SGD to converge faster when the objective under consideration is in some technical sense ``easier'' than a worst-case objective\footnote{In this paper we use ``adaptive'' in a statistical sense of the word: an algorithm is adaptive if it automatically adjusts itself to some unknown parameter, such as the variance of the gradients. This is different from the idea of using a different learning rate for each dimension of an optimization problem that was also popularized by \cite{duchi10adagrad}.}, specifically when the variance of the loss from one example to another is small \citep{li2019convergence,ward2019adagrad}. The success of AdaGrad has inspired a huge number of related algorithms, including notably the Adam algorithm \citep{kingma2014adam}.

One of the key improvements added to AdaGrad by Adam is the use of \emph{momentum}. Inspired by the heavy-ball and acceleration algorithms in convex optimization \citep{polyak1964some, nesterov1983method}, momentum attempts to improve the convergence rate on non-convex objectives by modifying the update to have the form:
\begin{align*}
    \bm_{t} &= \beta \bm_{t-1} + (1-\beta) \nabla f(\bw_t, \xi_t)\\
    \bw_{t+1} &= \bw_t - \eta_t \bm_t
\end{align*}
Intuitively, the value $\bm$ is holding a running average of the past gradient values, and the hope is that this style of update may provide some kind of better stability or conditioning that enables improvements over the base SGD. Momentum has had dramatic empirical success, but although prior analyses have considered momentum updates \citep{reddi2018on,zaheer2018adaptive}, none of these have shown a strong theoretical benefit in using momentum, as their bounds do not improve on (\ref{eqn:sgd}).

Finally, a third popular modification is the use of \emph{normalized} updates. For example, the LARS \citep{you2017large} and LAMB \citep{you2019reducing} optimizers use updates similar to:
\begin{align}
    \bm_{t} &= \beta \bm_{t-1} + (1-\beta) \nabla f(\bw_t, \xi_t)\label{eqn:normalized1}\\
    \bw_{t+1} &= \bw_t - \eta_t \frac{\bm_t}{\|\bm_t\|}\label{eqn:normalized2}
\end{align}
Intuitively, this style of update attempts to capture the idea that in non-convex objectives, unlike convex ones, the \emph{magnitude} of the gradient provides less information about the value of the function, while the \emph{direction} still indicates the direction of steepest descent. However, all analyses of normalized updates we know of (e.g. \cite{you2017large, you2019reducing,  hazan2015beyond}) require the variance of the gradient oracle to be very small, or, equivalently, for the algorithm to make use of an extremely large batch-size in order to achieve any convergence guarantees. Intuitively, this requirement arises because the normalization can inflate very small errors: even if $m_t = \nabla F(\bw_t)+\zeta$ for some small error $\zeta$, it might be that $\frac{m_t}{\|m_t\|}$ is actually very far from $\frac{\nabla F(\bw_t)}{\|\nabla F(\bw_t)\|}$. Nevertheless, this update has also been used to empirically accelerate training neural networks.

From a more theoretical perspective, a recent approach to going beyong the SGD rate (\ref{eqn:sgd}) is to add some additional structural assumptions to the loss. One appealing assumption is \emph{second-order smoothness}, in which we assume that the third derivative of $F$ has bounded operator norm. Equivalently, for all $\bw$, and $\bz$, we have
\begin{align}
    \|\nabla^2 F(\bw) - \nabla^2 F(\bz)\|_{\text{op}} \le \rho \|\bw-\bz\|\label{eqn:secondsmooth}\tag{A4}
\end{align}
Many previous works have achieved improved results by utilizing this assumption in concert with stronger oracles, such as evaluating two gradients per each example or evaluating Hessian-vector products \citep{allen2018natasha, tripuraneni2018stochastic}. However, using our same stochastic gradient oracle model and this second-order smoothness assumption, it was recently proven \citep{fang2019sharp} that a variant of SGD can achieve a convergence rate of
\begin{align*}
    \|\nabla F(\bw)\|\le O\left(\frac{\text{polylog(d)}}{T^{2/7}}\right)
\end{align*}
where $d$ is the dimension of the problem.\footnote{In fact, this result provided a convergence to a \emph{local minimum}, which is stronger than a critical point.} This break-through result shows that even for fairly high-dimensional problems, SGD can obtain faster convergence than the initial analysis suggests. However, in modern deep learning architectures, $d$ can be on the order of billions \citep{radford2019language}. In this regime, it may easily hold that the logarithmic term is large enough that this new analysis of SGD does not actually suggest improved performance over the previous $O(1/T^{1/4})$ rate. To deal with extremely high-dimensional regimes, we would like a convergence rate that is completely dimension-free.

In a somewhat orthogonal direction, several algorithms have been proposed based on \emph{variance reduction} \citep{johnson2013accelerating}. Recently, these have provided provable improvements over SGD's convergence rate, finding a point with $\E[\|\nabla F(\bw)\|]\le O(1/T^{1/3})$ after $T$ iterations \citep{fang2018spider, zhou2018stochastic}. Unfortunately, these algorithms require \emph{two} gradient evaluations for every i.i.d. sample $\bxi$. Our model requires each $\bxi$ to be used only once, and so forbids this operation. Moreover, it turns out to be surprisingly hard to implement this two-gradient model efficiently in modern machine learning frameworks. Nevertheless, more recent work in this area has produced algorithms whose updates are very similar to the (unnormalized) momentum update \citep{cutkosky2019momentum, tran2019hybrid}, suggesting that momentum may operate by reducing variance in SGD.

Finally, another technique that also suggests a connection to momentum in reducing variance is \emph{implicit gradient transport} \citep{arnold2019reducing}. Implicit gradient transport is a very recent discovery that provably reduces the variance in gradient estimates in the special case that the Hessian of the objective is \emph{constant} (e.g. if the objective is a linear regression problem). The original presentation considers the following version of momentum:
\begin{align*}
    \bm_t &= \frac{t}{t+1}\bm_{t-1} + \frac{1}{t+1}\nabla f(\bw_t+ t(\bw_t - \bw_{t-1}),\bxi_t)\\
    \bw_{t+1} &= \bw_t - \eta_t \bm_t
\end{align*}
To gain some intuition for why this is a good idea when the Hessian is constant, suppose we have $\bm_{t-1}=\nabla F(\bw_{t-1}) + X$ where $X$ is some mean-zero random variable with variance $\sigma^2/t$. Let us also write $\nabla f(\bw_t+ t(\bw_t - \bw_{t-1}),\bxi_t)=\nabla F(\bw_t+ t(\bw_t - \bw_{t-1}))+ Y$ where $Y$ is also a mean-zero random variable with variance $\sigma^2$.  Then since the Hessian is constant, we have:
\begin{align*}
    \nabla F(\bw_{t-1}) &= \nabla F(\bw_t) + \nabla^2 F (\bw_{t-1} - \bw_t)\\
    \nabla F(\bw_t+ t(\bw_t - \bw_{t-1}))&=t\nabla^2 F(\bw_t - \bw_{t-1})+  \nabla F(\bw_t)\\
    \bm_t  &= \nabla F(\bw_t) + \frac{t X + Y}{t+1}
\end{align*}
Notice that $ \frac{t X + Y}{t+1}$ has variance $\sigma^2/(t+1)$, so that by induction, we will have for all $t$ that $\bm_t$ is an unbiased estimate of $\nabla F(\bw_t)$ with variance only $\sigma^2/(t+1)$. This technique shows great promise, but the current theoretical analysis is limited to convex quadratic objectives.

In this paper we address some of these preceding issues. First, we show that the normalized stochastic gradient descent update with momentum does \emph{not} require small variance or large batches in order to match the convergence rate of SGD. Next, we tackle the case of second-order smooth objectives. For this, we introduce a further modification of the momentum update based on the implicit gradient transport idea. We show that combining this newer kind of momentum with normalized gradient descent enables a dimension-free convergence rate of $O(1/T^{2/7})$. Moreover, we feel that our new analysis is substantially more straightforward than prior work, as demonstrated by the number of pages required to prove our results. Finally, we propose an adaptive version of our algorithm and show that this final algorithm's convergence guarantee automatically improves when the stochastic gradient oracle has small variance. We hope our results can shed some light in the empirical success of momentum in training deep networks.

In addition to the theoretical discussion, we also demonstrate effectiveness of our method \night (pronounced ``night") on popular deep learning tasks. We show comparisons of our method on 1) BERT pretraining and 2) Resnet-50. Adam is typically used to achieve state of the art results on BERT task \citep{devlin-etal-2019-bert} whereas SGD is used for Resnet-50 on Imagenet dataset since Adam fails to perform well on it \citep{you2019reducing, Anil2019MemoryEA}. We show comparison with the stronger baseline in each case. Finally, since momentum is the only significant slot variable kept, our method incurs greatly reduced memory overhead compared to other methods such as Adam. This is a significant practical advantage due to the continuing trend of increasingly large models \citep{Shazeer2018AdafactorAL, Anil2019MemoryEA}. 

The rest of this paper is organized as follows: in Section \ref{sec:normalized}, we provide our first result showing that momentum can be used to ``rescue'' normalized SGD in the high-variance/small-batch regime. In Section \ref{sec:igt}, we expand upon this theory to show that incorporating a momentum update inspired by \cite{arnold2019reducing} allows for improved convergence rates when the objective is second-order smooth, and we proceed to incorporate adaptivity into our analysis in Section \ref{sec:adaptive}. In Section \ref{sec:experimental} we describe our implementation and experimental study, and we provide some concluding remarks and open problems in Section \ref{sec:conclusion}.

\section{Normalized SGD with Momentum}\label{sec:normalized}
In this section, we analyze the normalized SGD with momentum update given by equations (\ref{eqn:normalized1}) and (\ref{eqn:normalized2}). The result is Theorem \ref{thm:normalizedsgd}, which shows that the addition of momentum allows normalized SGD to match the optimal convergence rate obtained by ordinary SGD, without requiring any large batch sizes. To gain some intuition for why the momentum is necessary in this case, consider a one-dimensional optimization scenario in which $\nabla f(\bw, \xi) = p$ with probability $1-p$ and $p-1$ with probability $p$ for some $p\in(0,1/2)$. Then we clearly have $\nabla F(\bw)=\E[\nabla f(\bw,\xi)]=0$, so that intuitively the ``correct'' thing for the algorithm to do is to not move, at least on average. Unfortunately, we have $\E\left[\frac{\nabla f(\bw, \xi)}{\|\nabla f(\bw,\xi)\|}\right]=1-2p>0$, so that the algorithm will be biased to moving \emph{away} from this critical point. The problem here was that, without momentum, the error in the gradient estimate before normalization is much larger than the true value of the gradient. In order to fix this, one must decrease the variance in the stochastic gradient oracle by employing a very large batch size, usually on the order of $T$ \citep{you2019reducing,hazan2015beyond}. Our Theorem \ref{thm:normalizedsgd} avoids this requirement.

\begin{restatable}{Theorem}{thmnormalizedsgd}\label{thm:normalizedsgd}
Suppose $F$ and $\mathcal{D}$ satisfy the assumptions (\ref{eqn:smoothness}), (\ref{eqn:positive}) and (\ref{eqn:variance}). Let $\bw_1$ be some given initial point and set $\bm_1=\nabla f(\bw_1, \xi_1)$. Let $R$ be any upper bound on $F(\bw_1)$. Set $\alpha = \min\left(\frac{\sqrt{RL}}{\sigma \sqrt{T}}, 1\right)$ and $\eta = \frac{\sqrt{R\alpha}}{\sqrt{TL}}$ Then let $\bw_t$ be the iterates produced by the recurrences (\ref{eqn:normalized1}) and (\ref{eqn:normalized2}) with $\eta_t=\eta$ and $\beta_t = 1-\alpha$ for all $t$. Then the average norm of $\|\nabla F(\bw_t)\|$ satisfies:
\begin{align*}
    \frac{1}{T}\sum_{t=1}^T\E\left[ \|\nabla F(\bw_t)\|\right]&\le\frac{29\sqrt{RL}}{\sqrt{T}} +  \frac{21\sqrt{\sigma}(RL)^{1/4}}{T^{1/4}}   + \frac{8\sigma }{\sqrt{RLT}}
\end{align*}
\end{restatable}

Before proving this Theorem, we provide a general technical Lemma about each step of normalized SGD:
\begin{restatable}{Lemma}{thmonestepnormalized}\label{thm:onestepnormalized}
uppose $F$ and $\mathcal{D}$ satisfy the assumptions (\ref{eqn:smoothness}) and (\ref{eqn:variance}). Let $\bw_1$ be some initial point, and consider the updates:
\begin{align*}
    \bw_{t+1} = \bw_t - \eta_t \frac{\hat g_t}{\|\hat g_t\|}
\end{align*}
where $\hat g_t$ is generated by some arbitrary process or algorithm. Let $\hat \epsilon_t = \hat g_t - \nabla F(\bw_t)$. Then we have
\begin{align*}
    F(\bw_{t+1})-F(\bw_t)&\le -\frac{\eta_t}{3}\|\nabla F(\bw_t)\| + \frac{8\eta_t}{3}\|\hat \epsilon_t\| +\frac{L\eta_t^2}{2}
\end{align*}
In particular, if $F$ also satisfies (\ref{eqn:positive}) and $\eta_t$ is a constant $\eta$ for all $t$, we have:
\begin{align*}
    \sum_{t=1}^T \|\nabla F(\bw_t)\|&\le \frac{3F(\bw_1)}{\eta} + \frac{3L T\eta}{2}+8\sum_{t=1}^T\|\hat \epsilon_t\|
\end{align*}
\end{restatable}

\begin{proof}
Since $F$ is $L$-smooth, we have
\begin{align*}
    F(\bw_{t+1}) -F(\bw_t)&\le\langle \nabla F(\bw_t), \bw_{t+1}- \bw_t\rangle +\frac{L}{2}\|\bw_{t+1}-\bw_t\|^2\\
    &=-\eta_t \frac{\langle \nabla F(\bw_t), \hat g_t\rangle}{\|\hat g_t\|} + \frac{L\eta_t^2}{2}
\end{align*}
Now consider two cases, either $\|\nabla F(\bw_t)\|\ge 2\|\hat \epsilon_t\|$ or not. In the former case, we have
\begin{align*}
    -\frac{\langle \hat g_t, \nabla F(\bw_t)\rangle}{\|\hat g_t\|}&\le-\frac{\|\nabla F(\bw_t)\|^2 + \langle \nabla F(\bw_t), \hat\epsilon_t\rangle}{\|\nabla F(\bw_t) + \hat \epsilon_t\|}\\
    &\le -\frac{\|\nabla F(\bw_t)\|^2}{2\|\nabla F(\bw_t) + \hat \epsilon_t\|}\\
    &\le -\frac{\|\nabla F(\bw_t)\|}{3}\\
    &\le -\frac{\|\nabla F(\bw_t)\|}{3} + \frac{8\|\hat \epsilon_t\|}{3}
\end{align*}
In the latter case, we have
\begin{align*}
 -\frac{\langle \hat g_t, \nabla F(\bw_t)\rangle}{\|\hat g_t\|}&\le \|\nabla F(\bw_t)\|\\
&= -\frac{\|\nabla F(\bw_t)\|}{3} + \frac{4\|\nabla F(\bw_t)\|}{3}\\
&\le -\frac{\|\nabla F(\bw_t)\|}{3} + \frac{8\|\hat \epsilon_t\|}{3}
\end{align*}
So that the same bound holds in both cases. Putting everything together now proves the first statement of the Lemma. The second statement follows by summing over $T$, observing that the left-hand-side of the equation telescopes, and rearranging.
\end{proof}

By this Lemma, if we can show that our choices for $\beta$ and $\eta$ result in a small value for $\sum_{t=1}^T\|\hat \epsilon_t\|$, then setting $\eta$ appropriately will prove Theorem \ref{thm:normalizedsgd}. We carry out this agenda in more detail below:
\begin{proof}[Proof of Theorem \ref{thm:normalizedsgd}]
Define $\hat \epsilon_t = \bm_t - \nabla F(\bw_t)$ and define $\epsilon_t = \nabla f(\bw_t,\xi_t)-\nabla F(\bw_t)$. Notice that by our assumptions, we have
\begin{align*}
    \E[\|\epsilon_t\|^2]&\le \sigma^2\\
    \E[\langle \epsilon_i,\epsilon_j\rangle]&=0\ \text{ for }i\ne j
\end{align*}
Further, define $S(a,b) = \nabla F(a)-\nabla F(b)$. By smoothness, we must have $\|S(\bw_t,\bw_{t+1})\|\le L\|\bw_t-\bw_{t+1}\|=\eta L$ for all $t$.
With this notation, we have the following recursive formulation for any $t\ge 1$:
\begin{align*}
    \bm_{t+1} &= (1-\alpha) (\nabla F(\bw_t) + \hat \epsilon_t) + \alpha \nabla f(\bw_{t+1},\xi_{t+1})\\
    &=\nabla F(\bw_{t+1}) + (1-\alpha)(S(\bw_t,\bw_{t+1})+\hat \epsilon_t) + \alpha\epsilon_{t+1}\\
    \hat \epsilon_{t+1}&= (1-\alpha)S(\bw_t,\bw_{t+1})+(1-\alpha)\hat \epsilon_t + \alpha \epsilon_{t+1}
\end{align*}
Now, we unravel the recursion for $t$ iterations:
\begin{align*}
    \hat \epsilon_{t+1}&=(1-\alpha)^t\hat \epsilon_1 +\alpha\sum_{\tau=0}^{t-1} (1-\alpha)^\tau \epsilon_{t+1-\tau}+ (1-\alpha)\sum_{\tau=0}^{t-1}(1-\alpha)^\tau S(\bw_{t-\tau},\bw_{t+1-\tau})
\end{align*}
Next, take the magnitude of both sides and use triangle inequality:
\begin{align*}
    \|\hat \epsilon_{t+1}\|&\le (1-\alpha)^t\|\epsilon_1\| +\alpha\left\|\sum_{\tau=0}^{t-1} (1-\alpha)^\tau \epsilon_{t+1-\tau}\right\|+ (1-\alpha)\eta L \sum_{\tau=0}^{t-1}(1-\alpha)^\tau
\end{align*}
Where in the above we have observed that $\hat\epsilon_1=\epsilon_1$. Now take expectation, and use the fact that any two $\epsilon_i$ are uncorrelated and apply Jensen inequality to obtain:
\begin{align*}
    \E\left[\|\hat \epsilon_{t+1}\|\right]&\le (1-\alpha)^t\E\left[\|\epsilon_1\|\right] + \alpha \sqrt{\sum_{\tau=0}^{t-1} (1-\alpha)^{2\tau}\sigma^2}+ (1-\alpha)\eta L\sum_{\tau=0}^{t-1}(1-\alpha)^\tau\\
    &=(1-\alpha)^t\sigma +\frac{ \alpha\sigma }{\sqrt{1-(1-\alpha)^2}}+\frac{\eta L}{\alpha}\\
    &\le (1-\alpha)^t\sigma +\sqrt{\alpha}\sigma +\frac{\eta L}{\alpha}\\
    \E\left[\sum_{t=1}^T \|\hat \epsilon_t\|\right]&\le\frac{\sigma}{\alpha}+ T\sqrt{\alpha} \sigma + \frac{T  \eta L}{\alpha}
\end{align*}
Applying Lemma \ref{thm:onestepnormalized} now yields:
\begin{align*}
    \sum_{t=1}^T \E\left[\|\nabla F(\bw_t)\|\right]&\le \frac{3F(\bw_1)}{\eta} + \frac{3L T\eta}{2}+8\sum_{t=1}^T\|\hat \epsilon_t\|\\
    &\le \frac{3R}{\eta} + \frac{3TL\eta}{2} + \frac{8\sigma}{\alpha} +8T\sqrt{\alpha} \sigma+\frac{8 T L\eta}{\alpha}\\
    &\le \frac{3R}{\eta} + \frac{8\sigma}{\alpha} +8T\sqrt{\alpha} \sigma+\frac{10 T L\eta}{\alpha}
\end{align*}
Where we have used $\alpha\le 1$. Set $\alpha = \min\left(\frac{\sqrt{RL}}{\sigma \sqrt{T}}, 1\right)$ and $\eta = \frac{\sqrt{R\alpha}}{\sqrt{TL}}$. Observe from the setting of $\eta$ that we have:
\begin{align*}
    \sum_{t=1}^T \E\left[\|\nabla F(\bw_t)\|\right]&\le\frac{13 \sqrt{RLT}}{\sqrt{\alpha}} + 8\sqrt{\alpha}\sigma T + \frac{8\sigma}{\alpha}
\end{align*}
Now suppose $\alpha=1$. This implies that $\sigma \le \frac{\sqrt{RL}}{\sqrt{T}}$. Therefore the RHS of the above expression is bounded by $29\sqrt{RLT}$. On the other hand, if $\alpha = \frac{\sqrt{RL}}{\sigma \sqrt{T}}$, the RHS is bounded by
\begin{align*}
    21\sqrt{\sigma}(RL)^{1/4} T^{3/4}  + \frac{8\sigma \sqrt{T}}{\sqrt{RL}}
\end{align*}
Adding these expressions yields the desired result.
\end{proof}

\section{Faster Convergence with Second-Order Smoothness}\label{sec:igt}
Now that we have seen a simpler example of how to analyze the convergence of normalized SGD with momentum in the non-convex setting, we can proceed to consider the more advanced case that $F$ is second-order smooth. Notice that in the proof of Theorem \ref{thm:normalizedsgd}, we made use of the quantity $S(a,b) = \nabla F(a)-\nabla F(b)$, which satisfies $\|S(a,b)\|\le L\|a-b\|$. In this section, we will need a related quantity:
\begin{align*}
Z(a,b) = \nabla F(a) - \nabla F(b) - \nabla^2F(b)(a-b)
\end{align*}
Assuming $F$ satisfies (\ref{eqn:secondsmooth}), Taylor's theorem implies that $Z$ satisfies $\|Z(a,b)\|\le \rho \|a-b\|^2$. The improved dependency on $\|a-b\|$ from linear to quadratic will allow us to achieve a faster convergence rate.

Without further ado, we present our algorithm \night in Algorithm \ref{alg:normalized} below:
\begin{algorithm}
   \caption{Normalized SGD with Implicit Gradient Transport (\night - pronounced ``night'')}
   \label{alg:normalized}
   \begin{algorithmic}
      \STATE{\bfseries Input: } Initial Point $\bw_1$, learning rate $\eta$, momentum parameter $\beta$.
      \STATE Set $\bm_1 = \nabla f(\bx_1, \xi_1)$.
      \STATE Set $\bw_2 = \bw_1 - \eta \frac{\bm_1}{\|\bm_1\|}$.
      \FOR{$t=2\dots T$}
      \STATE Set $\bx_t = \bw_t + \frac{\beta}{1-\beta}(\bw_t - \bw_{t-1})$.
      \STATE Set $\bm_t = \beta_t \bm_{t-1} + (1-\beta_t) \nabla f(\bx_t ,\xi_t)$.
      \STATE Set $\bw_{t+1} = \bw_t - \eta \frac{\bm_t}{\|\bm_t\|}$.
      \ENDFOR
   \end{algorithmic}
\end{algorithm}
The auxiliary variable $\bx_t$ is introduced for notational convenience - it can be recomputed every iteration from $\bw_t$ and $\bw_{t-1}$. Notice that, with $\beta_t = \frac{t}{t+1}$, the update for $\bm_t$ corresponds exactly to the implicit gradient transport update proposed by \cite{arnold2019reducing}. We will instead use the different setting $\beta = 1 - O(T^{-4/7})$. We show that this value of $\beta$, combined with the normalized SGD update, enables faster convergence on second-order smooth non-convex objectives. This significantly extends the prior results of \cite{arnold2019reducing}, which hold only on convex quadratic objectives with constant Hessian, and helps realize the theoretical potential of the implicit gradient transport idea.

\begin{restatable}{Theorem}{thmigt}\label{thm:igt}
Suppose $f$ and $\mathcal{D}$ satisfies the assumptions (\ref{eqn:smoothness}), (\ref{eqn:positive}), (\ref{eqn:variance}), and (\ref{eqn:secondsmooth}). Let $\bw_1$ be an arbitrary starting point, and set $\bm_1 = \nabla f(\bw_1, \xi_1)$. Let $R\ge F(\bw_1)$ and set $\eta=\min\left( \frac{R^{5/7}}{T^{5/7}\rho^{1/7}\sigma^{4/7}}, \sqrt{\frac{R}{TL}}\right)$ and $\alpha=\min\left(\frac{R^{4/7}\rho^{2/7}}{T^{4/7}\sigma^{6/7}}, 1\right)$. Suppose $\bw_1,\dots,\bw_T$ are the iterates produced by Algorithm \ref{alg:normalized} with $\eta_t=\eta$ and $\beta_t = 1-\alpha$ for all $t$. Then the average expected gradient norm satisfies:
\begin{align*}
   \frac{1}{T}\sum_{t=1}^T\E\left[\|\nabla F(\bw_t)\|\right]\le
   \frac{5\sqrt{RL}}{\sqrt{T}}  +\frac{8\sigma^{13/7}}{R^{4/7}\rho^{2/7}T^{3/7}}+ \frac{27R^{2/7}\rho^{1/7}\sigma^{4/7}}{T^{2/7}}
\end{align*}
In particular, the dependence on $T$ is $O(1/T^{2/7})$.
\end{restatable}

\begin{proof}
Our proof begins very similarly to that of Theorem \ref{thm:normalizedsgd}. We define $\hat \epsilon_t = m_t - \nabla F(\bw_t)$, and $\epsilon_t = \nabla f(\bx_t,\xi_t) - \nabla F(\bx_t)$. Then we have the following recursive formulation:
\begin{align*}
    \bm_{t+1} &= (1-\alpha)(\nabla F(\bw_t) + \hat \epsilon_t)+\alpha\nabla f(\bx_{t+1}, \xi_{t+1})\\
    &=(1-\alpha)(\nabla F(\bw_{t+1})+\nabla^2F(\bw_{t+1})(\bw_t-\bw_{t+1})) + (1-\alpha)(Z(\bw_t,\bw_{t+1})+\hat\epsilon_t)+\alpha\nabla F(\bw_{t+1})\\ 
    &\quad+\alpha\left(\frac{1-\alpha}{\alpha}\nabla^2F(\bw_{t+1})(\bw_{t+1}-\bw_t)\right)+\alpha(Z(\bx_{t+1}, \bw_{t+1})+\epsilon_{t+1})\\
    &=\nabla F(\bw_{t+1}) + (1-\alpha)Z(\bw_t, \bw_{t+1}) + \alpha Z(\bx_{t+1}, \bw_{t+1}) + (1-\alpha)\hat \epsilon_t + \alpha\epsilon_{t+1}\\
    \hat\epsilon_{t+1}&=(1-\alpha)\hat \epsilon_t+ (1-\alpha)Z(\bw_t, \bw_{t+1}) + \alpha Z(\bx_{t+1}, \bw_{t+1})  + \alpha\epsilon_{t+1}
\end{align*}
Here we have already used the key insight of implicit gradient transport, as well as our second-order smoothness assumption. The carefully constructed equation for $\bx_t$ is such that the $\nabla^2 F(\bw_{t+1})$ terms cancel out, and we keep track of the error introduced by having a non-constant Hessian in the $Z$ function. Next, we unroll the recursion to obtain:
\begin{align*}
    \hat \epsilon_{t+1} &= (1-\alpha)^t \hat \epsilon_1 + \alpha\sum_{\tau=0}^{t-1}(1-\alpha)^\tau \epsilon_{t+1-\tau} + (1-\alpha)\sum_{\tau=0}^{t-1} (1-\alpha)^\tau Z(\bw_{t-\tau}, \bw_{t+1-\tau}) + \alpha\sum_{\tau=0}^{t-1} (1-\alpha)^\tau Z(\bx_{t+1}, \bw_{t+1})
\end{align*}
Now, recall that $Z(a,b)\le \rho\|a-b\|$. This implies:
\begin{align*}
    \alpha \|Z(\bx_t,\bw_t)\|&\le \rho \frac{(1-\alpha)^2\eta^2}{\alpha}\le \rho \frac{(1-\alpha)\eta^2}{\alpha}\\
    (1-\alpha)\|Z(\bw_{t}, \bw_{t+1})\|&\le(1-\alpha)\rho \eta^2\le \rho \frac{(1-\alpha)\eta^2}{\alpha}
\end{align*}
Now just as in the proof of Theorem \ref{thm:normalizedsgd}, we take magnitudes, use triangle inequality, and take expectations, but this time we use the bounds on $Z$:
\begin{align*}
    \E[\|\hat\epsilon_{t+1}\|]&\le (1-\alpha)^t\sigma + \alpha \sqrt{\sum_{\tau=0}^{t-1}(1-\alpha)^{2\tau}\sigma^2}+2\frac{\eta^2\rho}{\alpha}\sum_{\tau=0}^{t-1}(1-\alpha)^{\tau+1}\\
    &\le (1-\alpha)^t\sigma + \frac{\alpha\sigma}{\sqrt{1-(1-\alpha)^2}} + 2\frac{\eta^2\rho(1-\alpha)}{\alpha^2}\\
    &\le (1-\alpha)^t\sigma + \sqrt{\alpha}\sigma + 2\frac{\eta^2\rho(1-\alpha)}{\alpha^2}
\end{align*}
Where we have used $\hat \epsilon_1 = \epsilon_1$ in the first line. Now sum over $t$ to obtain:
\begin{align*}
    \sum_{t=1}^T \E[\|\hat \epsilon_t\|]&\le \frac{\sigma(1-\alpha^{T})}{\alpha} + T\sqrt{\alpha}\sigma + 2\frac{T\eta^2 \rho(1-\alpha)}{\alpha^2}
\end{align*}
Next, we apply Lemma \ref{thm:onestepnormalized}:
\begin{align*}
    \sum_{t=1}^T \E[\|\nabla F(\bw_t)\|]&\le \frac{3F(\bw_1)}{\eta} + \frac{3L T\eta}{2}+\frac{8\sigma(1-\alpha^{T})}{\alpha}+8T\sqrt{\alpha}\sigma +16\frac{T\eta^2 \rho(1-\alpha)}{\alpha^2}
\end{align*}
Now set $\eta=\min\left( \frac{R^{5/7}}{T^{5/7}\rho^{1/7}\sigma^{4/7}}, \sqrt{\frac{R}{TL}}\right)$, so that we have:
\begin{align*}
    \frac{3F(\bw_1)}{\eta} + \frac{3L T\eta}{2}&\le 5\sqrt{RLT} + 3R^{2/7}\rho^{1/7}\sigma^{4/7}T^{5/7}
\end{align*}

Further, set $\alpha = \min\left(\frac{R^{4/7}\rho^{2/7}}{T^{4/7}\sigma^{6/7}}, 1\right)$. Notice that in all cases we have $T\sqrt{\alpha}\sigma \le R^{2/7}\rho^{1/7}\sigma^{4/7}T^{5/7}$, so we can conclude,
\begin{align}
    \sum_{t=1}^T \E[\|\nabla F(\bw_t)\|]&\le 5\sqrt{RLT} + 11R^{2/7}\rho^{1/7}\sigma^{4/7}T^{5/7}+\frac{8\sigma(1-\alpha^{T})}{\alpha} +16\frac{T\eta^2 \rho(1-\alpha)}{\alpha^2}\label{eqn:midway}
\end{align}
Let us consider the case that $\alpha=1$. In this case, the last two terms in (\ref{eqn:midway}) are zero, so we have:
\begin{align}
    \sum_{t=1}^T \E[\|\nabla F(\bw_t)\|]&\le 5\sqrt{RLT} + 11R^{2/7}\rho^{1/7}\sigma^{4/7}T^{5/7}\label{eqn:case1}
\end{align}
Next, suppose $\alpha=\frac{R^{4/7}\rho^{2/7}}{T^{4/7}\sigma^{6/7}}$. Then instead we use the fact that $\eta= \frac{R^{5/7}}{T^{5/7}\rho^{1/7}\sigma^{4/7}}$ to refine (\ref{eqn:midway}) as:
\begin{align}
    \sum_{t=1}^T \E[\|\nabla F(\bw_t)\|]&\le 5\sqrt{RLT} + 11R^{2/7}\rho^{1/7}\sigma^{4/7}T^{5/7}+\frac{8\sigma}{\alpha} +16\frac{T\eta^2 \rho}{\alpha^2}\nonumber\\
    &5\sqrt{RLT} + 27R^{2/7}\rho^{1/7}\sigma^{4/7}T^{5/7}+\frac{8\sigma^{13/7}T^{4/7}}{R^{4/7}\rho^{2/7}}\label{eqn:case2}
\end{align}
Taking the maximum of (\ref{eqn:case1}) and (\ref{eqn:case2}) yields the desired convergence guarantee. 
\end{proof}

\section{Adaptive Algorithm}\label{sec:adaptive}
In the previous sections, we considered the case of \emph{constant} learning rate $\eta$ and momentum $\beta$ parameters. We showed that with appropriate settings, we can obtain fast convergence rates. In this section, we go further and consider varying $\eta$ and $\beta$ values. In particular, we will provide an \emph{adaptive} schedule in which $\eta$ and $\beta$ are adjusted on-the-fly based on the observed gradients. This adaptive schedule yields a convergence guarantee that automatically improves when $\sigma$ is small, \emph{without knowing $\sigma$}. In order to do this, we will need to sample \emph{two} gradients at each point $\bx_t$. The two gradient samples are independent, so this still fits within our stochastic gradient oracle model of computation. We denote the second independent gradient estimate evaluated at the point $\bx_t$ as $\nabla f(\bx_t, \xi_t')$. Finally, we will add another assumption: the function $F$ (or at least the values of $F(\bx_t)$) should be bounded above by some constant $M$. With this notation, our adaptive algorithm is presented in Algorithm \ref{alg:adaptive} and analyzed in Theorem \ref{thm:adaptive} below.

\begin{algorithm}
   \caption{Adaptive Normalized SGD with Momentum}
   \label{alg:adaptive}
\begin{algorithmic}
   \STATE{\bfseries Input: } Initial point $\bw_1$, lipschitz bound $\G$.
   \STATE Set $C = \sqrt{\frac{7}{26\G^{6/7}}}$ and $D = C^{-14/3}$
   \STATE Set $\bw_{0} = \bw_1$.
   \STATE Set $\bm_{0}= 0$.
   \STATE Set $G_1=3\G^2+D$.
   \STATE Set $\eta_0 = \frac{C}{D^{2/7}}$.
   \FOR{$t=1\dots T$}
   \STATE Set $\eta_t = \frac{C}{(G_t^2 (t+1)^3)^{1/7}}$.
   \STATE Set $\alpha_t = \frac{1}{t\eta_{t-1}^2 G_{t-1}}$.
   \STATE Set $\beta_t = 1-\alpha_t$.
   \STATE Set $\bx_t=\bw_t + \frac{\beta_t}{1-\beta_t}(\bw_t - \bw_{t-1})$.
   \STATE Set $\bm_t = \beta_t \bm_{t-1} + (1-\beta_t) \nabla f(\bx_t, \xi_t)$.
   \STATE Set $G_{t+1} = G_t + \|\nabla f(\bx_t, \xi_t) - \nabla f(\bx_t, \xi_t')\|^2 + \G^2((t+1)^{1/4}-t^{1/4})$.
   \STATE Set $\bw_{t+1} = \bw_t - \eta_t \frac{\bm_t}{\|\bm_t\|}$.
   \ENDFOR
\end{algorithmic}
\end{algorithm}

\begin{restatable}{Theorem}{thmadaptive}\label{thm:adaptive}
Suppose that $f$ and $\mathcal{D}$ satisfies (\ref{eqn:smoothness}), (\ref{eqn:positive}), (\ref{eqn:variance}), and (\ref{eqn:secondsmooth}). Further, suppose $\|\nabla f(\bw, \xi)\|\le \G$ for all $\bw$ with probability 1, and that $F(\bw)\le M$ for all $\bw$ for some $M$. Then Algorithm \ref{alg:adaptive} guarantees:
\begin{align*}
    \frac{1}{T}\sum_{t=1}^T \E\left[\|\nabla F(\bw_t)\|\right]&\le \tilde O\left(\frac{1}{\sqrt{T}} + \frac{\sigma^{4/7}}{T^{2/7}}\right)
\end{align*}
where the $\tilde O$ notation hides constants that depend on $R$, $G$, and $M$ and a factor of $\log(T)$.
\end{restatable}
We defer the proof to Appendix \ref{sec:proof}.

\section{Experiments}\label{sec:experimental}
Now, we turn to experimental evaluation of the proposed method \night on two popular large-scale deep learning benchmarks: BERT pretraining and ResNet-50. First we explain the setup and choice of hyperparameters for our method, and then elucidate both tasks and the details on the baseline used for each task.

\subsection{Setup}
We implemented our algorithm in the Tensorflow framework. For simplicity, we implemented a per-layer version of our algorithm, normalizing the gradients for each layer in the network, rather than normalizing the full gradient.
Taking our cue from defaults from previous empirical literature on momentum, we set the $\beta$ parameter to 0.9 for \night for both BERT and ResNet-50. For BERT, we stick with the learning rate schedule used for Adam in \cite{devlin-etal-2019-bert} i.e linear warmup and polynomial decay of $\eta_{t} = \eta_{0}*(1 - \frac{t}{T})$. Whereas for ResNet-50, we found that linear warmup and polynomical decay of $\eta_{t} = \eta_{0}*(1 - \frac{t}{T})^2$ worked best \citep{you2017large}. We performed a grid search on base learning rate $\eta_{0} \in [10^{-5}, 10^{-4}, 10^{-3}, 10^{-2}, 10^{-1}, 10^{0}]$ for both the tasks. In our implementation, we also scale the learning rate with the norm of the weights for each layer similar to \cite{you2017large}. We did not normalize gradients for bias, batch normalization and layer normalization parameters, and we scaled their learning rates by a factor of 1000. All our experiments were conducted on a TPUv3 architecture.

\subsection{BERT pretraining}
We replicate most of the setup created by \cite{devlin-etal-2019-bert} for our BERT baseline. Our text corpus is a concatenation of Wikipedia and BooksCorpus, with a total of roughly 3.3B words. We use examples of sequence length 512 with token masking rate of 0.15 and train the BERT-Base model with Masked Language Modeling target accuracy of 70.0\% for 500k steps. Fig \ref{fig:experiments} compares masked language modeling accuracy resulting with our method vs the baseline. We set \night base learning rate $\eta_{0}$ to 0.001 and batch size to 256. We replicate the hyperparameters used for our baseline method Adam exactly from \cite{devlin-etal-2019-bert} i.e. step size $\alpha=0.0001$, $\beta_{1}=0.9$, $\beta_{2}=0.99$.

\subsection{Resnet-50}
We train the ImageNet dataset with Resnet-50 architecture and compare our method against the commonly used SGD optimizer with momentum in Fig \ref{fig:experiments}. Others have observed that Adam/Adagrad fails to achieve accuracy attained by SGD on this task \citep{you2019reducing, Anil2019MemoryEA}. For our experiments, we set the batch size to 1024 and train the models for 90 epochs. We set the base learning rate $\eta_{0}$ for \night to 0.01. For the SGD baseline, we stick with the common practice of employing gradual warmup for 5 epochs up to a base learning rate of 0.4 and then divide it by 10 at 30-th, 60-th and 80-th epoch, as done in \cite{goyal2017accurate, He2015DeepRL}.


\begin{figure*}[ht!]
\centering
  \begin{subfigure}
     \centering
     \includegraphics[width=0.45\textwidth]{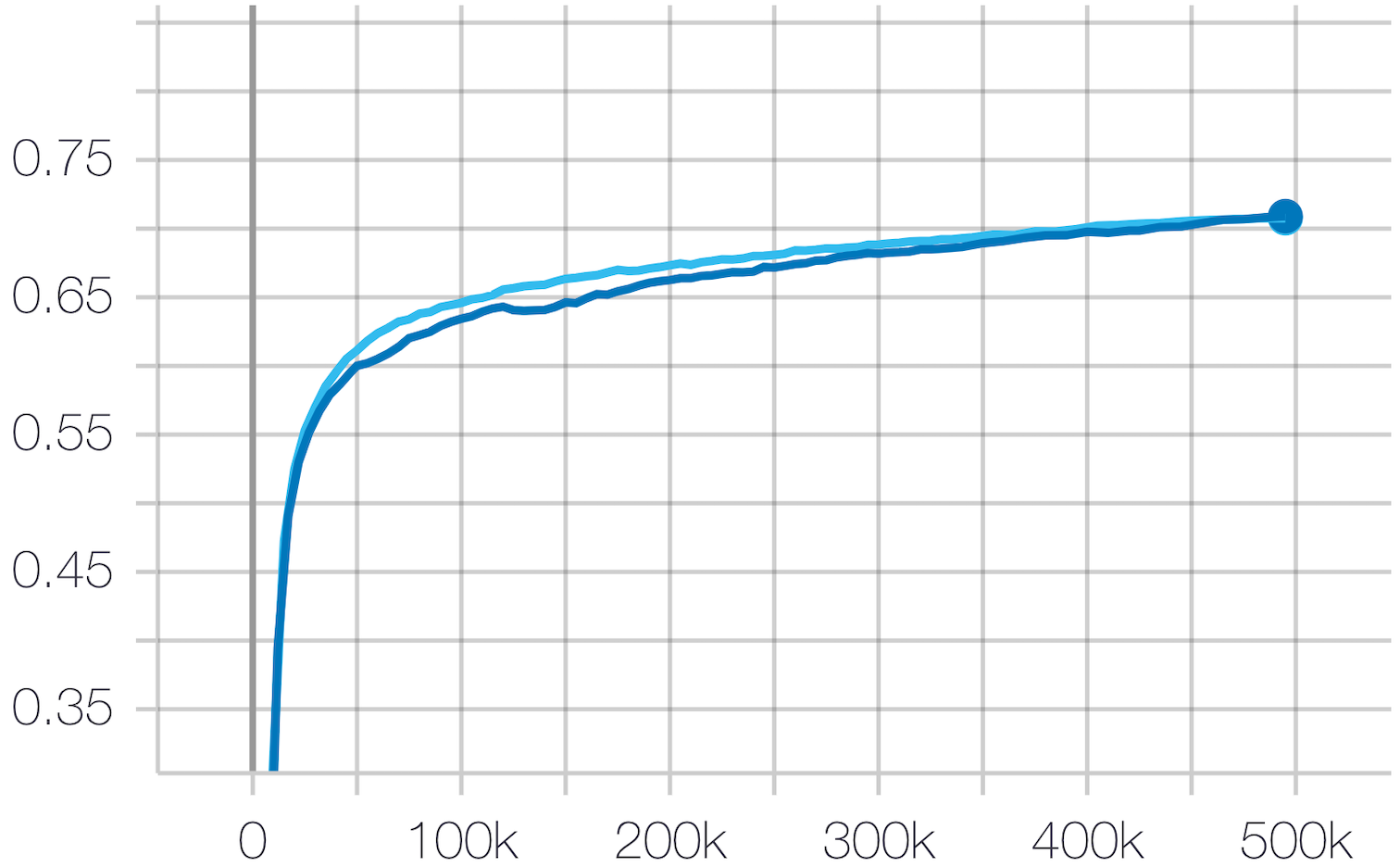}
     \label{fig:nigt_vs_adam}
    \qquad
  \end{subfigure}
      \begin{subfigure}
     \centering
     \includegraphics[width=0.45\textwidth]{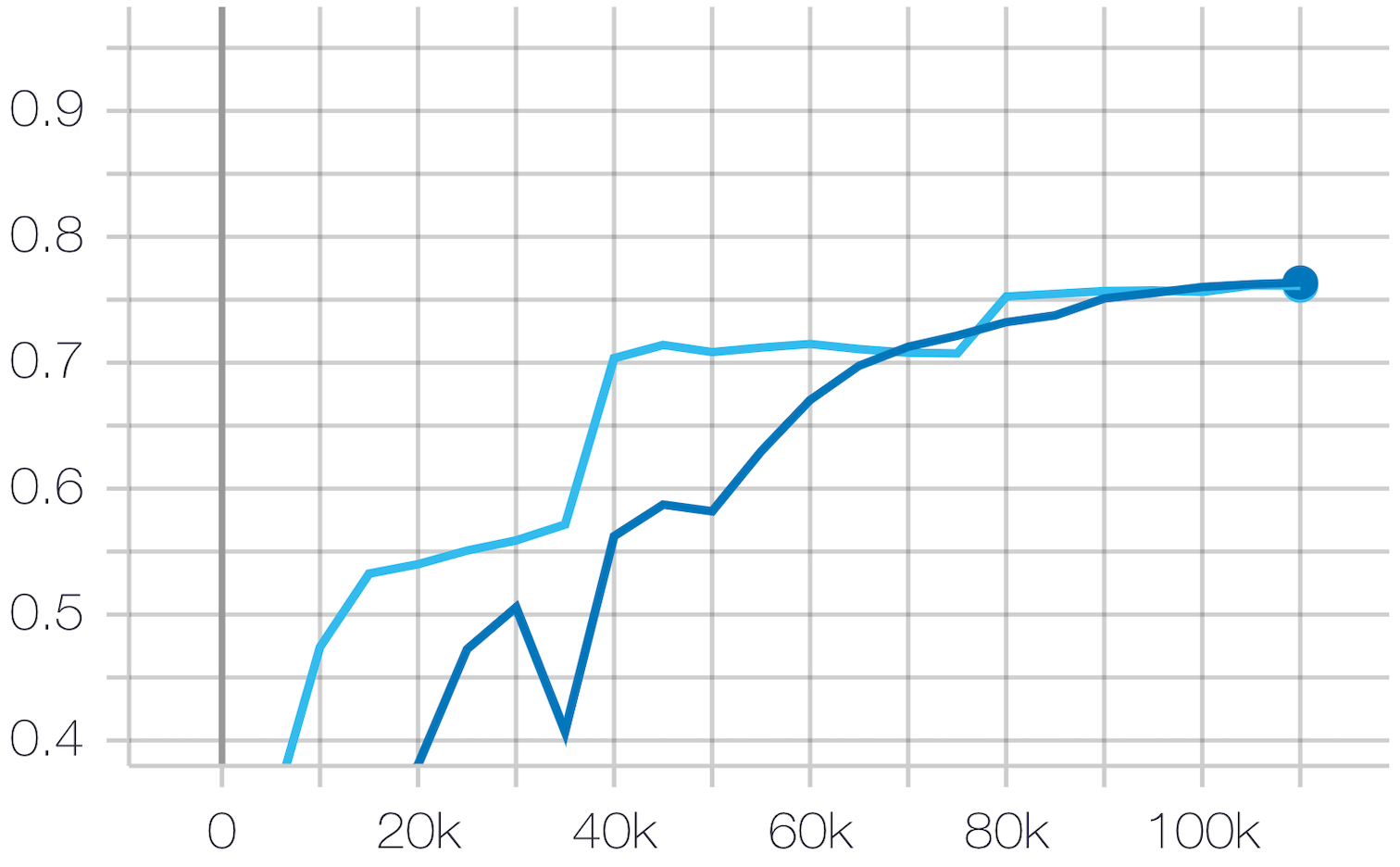}
     \label{fig:nigt_vs_momentum_resnet}
  \end{subfigure}
\caption{ (a) Masked language modeling validation accuracy comparisons for our method compared to Adam. Our method \night (dark blue), in the end, fares slightly better with 70.91 vs 70.76 for Adam (light blue). (b) Top-1 validation accuracy attained by Resnet-50 architecture over Imagenet dataset. Accuracy at the end of 90 epochs: \night (dark blue) - 76.37\% and SGD (light blue) - 76.2\%.
    }
  \label{fig:experiments}
\end{figure*}

\section{Conclusion and Future Work}\label{sec:conclusion}

We have presented a new analysis of the effects of momentum on normalized SGD, and showed that when the objective is second-order smooth, normalized SGD with momentum can find an $\epsilon$-critical point in $O(1/\epsilon^{3.5})$ iterations, matching the best-known rates in a dimension-free manner and surprisingly straightforward manner. Further, we provided an adaptive version of our algorithm whose convergence guarantee automatically improves when the underlying distribution has low variance. Finally, we provided an empirical evaluation showing that our algorithm matches the performance of the methods used to achieve state-of-the-art in training ResNet-50 on ImageNet and on BERT pretraining. Notably, these two tasks are often tackled with different optimizers (SGD and Adam respectively). Our method is one of the select few \citep{you2019reducing, Anil2019MemoryEA} which performs well on both simultaneously. Also, since \night does not maintain any second order statistic, it is significantly more memory efficient compared to popular optimizers such as Adam or LAMB.


Our results provide some intuition that can help explain the practical success of momentum in stochastic non-convex optimization. Prior analyses of momentum have failed to demonstrate significant theoretical benefits, but our results show that if one is willing to posit second-order smoothness, momentum may play a role in accelerating convergence. However, there are several open problems remaining. First, we suspect that the upper bound $M$ imposed on $F$ in our adaptive analysis is unnecessary, and also that the use of an extra gradient sample per iteration can be removed. Second, we note that in the noise-free case with second-order smoothness, the best-known dimension-free convergence rate is actually $O(1/\epsilon^{1.75})$ \citep{carmon2017convex} rather than the $O(1/\epsilon^2)$ achieved by gradient descent. Is it possible to design an adaptive algorithm that interpolates between this rate and $O(1/\epsilon^{3.5})$ in the noisy setting?

Finally, our implementation employs a few extra learning rate heuristics popular in practice, notably linear warm-up and polynomial decay \citep{devlin-etal-2019-bert}, as well as scaling the learning rate by the norm of the weights \citep{you2017large}. We found these heuristics to be better-performing than our adaptive algorithm, despite their lack of theoretical motivation. Understanding the principles underlying the success of these heuristics poses an interesting orthogonal question to understanding momentum, and we hope that our work inspires future investigation into these areas.

{\small
\bibliography{all}
\bibliographystyle{icml2020}
}


\clearpage
\appendix

\section{Proof of Theorem \ref{thm:adaptive}}\label{sec:proof}
In this section we provide the missing proof of Theorem \ref{thm:adaptive}, restated below:
\thmadaptive*

\begin{proof}
Notice from the definition of $\alpha$ that we always have
\begin{align*}
    \alpha_t &= \frac{1}{C^2t^{1/7}G_{t-1}^{3/7}}\\
    &\le \frac{1}{C^2D^{3/7}}\\
    &\le 1
\end{align*}
where we defined $G_0=D$. Thus $\alpha_t\le 1$ always.

We begin with the now-familiar definitions:
\begin{align*}
    \epsilon_t &= \nabla f(\bx_t, \xi_t) - \nabla F(\bx_t)\\
    \epsilon_t' &= \nabla f(\bx_t, \xi_t') - \nabla F(\bx_t)\\
    \hat \epsilon_t &= \bm_t - \nabla F(\bw_t)
\end{align*}
Notice that $\E[\langle \epsilon_i, \epsilon_j\rangle]=\sigma^2\delta_{i,j}$. Now we write the recursive formulation for $\hat\epsilon_{t+1}$:

\begin{align*}
    \bm_{t} &= (1-\alpha_t) \bm_{t-1} + \alpha_t \nabla f(\bx_t, \xi_t)\\
    &= (1-\alpha_t)(\nabla F(\bw_{t-1}) + \hat\epsilon_{t-1}) + \alpha_t (\nabla F(\bw_t) + \epsilon_t)\\
    &=\nabla F(\bw_t) + (1-\alpha_t)Z(\bw_{t-1}, \bw_t)+ \alpha_t Z(\bx_t, \bw_t) + (1-\alpha_t)\hat \epsilon_{t-1} + \alpha \epsilon_t\\
    \hat \epsilon_t&= (1-\alpha_t)Z(\bw_{t-1}, \bw_t) + \alpha_t Z(\bx_t, \bw_t)+ (1-\alpha_t)\hat \epsilon_{t-1} + \alpha_t \epsilon_t
\end{align*}
Unfortunately, it is no longer clear how to unroll this recurrence to solve for $\hat \epsilon_t$ in a tractable manner. Instead, we will take a different path, inspired by the potential function analysis of \cite{cutkosky2019momentum}. Start with the observations:
\begin{align*}
    \alpha_t \|Z(\bx_t,\bw_t)\|&\le \rho \frac{(1-\alpha_t)^2\eta_{t-1}^2}{\alpha_t}\le \rho \frac{(1-\alpha_t)\eta_{t-1}^2}{\alpha_t}\\
    (1-\alpha_t)\|Z(\bw_{t-1}, \bw_t)\|&\le(1-\alpha_t)\rho \eta_{t-1}^2\le \rho \frac{(1-\alpha_t)\eta_{t-1}^2}{\alpha_t}
\end{align*}

Define $K_t = \frac{1}{\alpha_t^2\eta_{t-1} G_t}$. Then we use $(a+b)^2\le (1+1/x)a^2+(1+x)b^2$ for all $x$ and the fact that $\epsilon_t$ is uncorrelated with anything that does not depend on $\xi_t$ to obtain:
\begin{align*}
    \E[K_t \|\hat\epsilon_t\|^2]&\le \E\left[ K_t(1+1/x)4\rho^2\frac{(1-\alpha_t)^2\eta_{t-1}^4}{\alpha_t^2}+ K_t(1+x)(1-\alpha_t)^2\|\epsilon_{t-1}\|^2+K_t\alpha_t^2 \|\epsilon_t\|^2\right]\\
    &\le \E\left[ K_t\rho^2\frac{8(1-\alpha_t)^2\eta_{t-1}^4}{\alpha_t^3}+ K_t(1-\alpha_t)\|\epsilon_{t-1}\|^2+K_t\alpha_t^2 \|\epsilon_t\|^2\right]
\end{align*}
where in the last inequality we have set $x=\alpha_t$. This implies:
\begin{align*}
    &\E[K_t\|\hat \epsilon_t\|^2 - K_{t-1}\|\hat \epsilon_{t-1}\|^2]\le  \E\left[ K_t\rho^2\frac{8(1-\alpha_t)^2\eta_{t-1}^4}{\alpha_t^3}+ (K_t(1-\alpha_t)-K_{t-1})\|\epsilon_{t-1}\|^2+K_t\alpha_t^2 \|\epsilon_t\|^2\right]\\
    &\le \E\left[ \rho^2\frac{8(1-\alpha_t)^2\eta_{t-1}^3}{\alpha_t^5G_t}+\frac{\|\epsilon_t\|^2}{G_t \eta_{t-1}}- \left(\frac{1}{\alpha_{t-1}^2G_{t-1}\eta_{t-2}} - \frac{1}{\alpha_t^2G_t\eta_{t-1}}+\frac{1}{\alpha_tG_t\eta_{t-1}}\right)\|\epsilon_{t-1}\|^2\right]
\end{align*}
Let $\delta_t = G_t - G_{t-1}$. Then we have $\E[\epsilon_t/\sqrt{G_t\eta_{t-1}}]=0= \E[\epsilon'_t/\sqrt{G_t\eta_{t-1}}]$. Therefore we have
\begin{align*}
    \E\left[\frac{\|\epsilon_t\|^2}{G_t}\right]&\le \E\left[\frac{\|\nabla f(\bx_t,\xi_t)\|^2}{G_t}\right]\\
    \E\left[\frac{\|\epsilon_t\|^2}{G_t}\right]&\le \E\left[\frac{\|\nabla f(\bx_t,\xi_t)-\nabla f(\bx_t, \xi_t')\|^2}{G_t}\right]
\end{align*}
so that we have
\begin{align*}
    \E\left[\frac{\|\epsilon_t\|^2}{G_t}\right]&\le\E\left[\frac{\delta_{t+1}}{G_t}\right]
\end{align*}
Now, observe that $\delta_{t+1}\le 2\G^2$, so that we have
\begin{align*}
    \E\left[\frac{\delta_{t+1}}{G_t\eta_{t-1}}\right]&= \E\left[\frac{\delta_{t+1}/\eta_{t-1}}{D+3\G^2+\sum_{\tau=1}^t \delta_t}\right]\\
    &\le \E\left[\frac{1}{\eta_T}\frac{\delta_{t+1}}{D+2\G^2+\sum_{\tau=1}^{t+1} \delta_t}\right]\\
    &\le \E\left[\frac{1}{\eta_T}\frac{\delta_{t+1}}{D+\sum_{\tau=1}^{t+1} \delta_t}\right]\\
    \sum_{t=2}^{T+1} \E\left[\frac{\|\epsilon_t\|^2}{G_t\eta_{t-1}}\right]&\le \E\left[\frac{1}{\eta_T}\log\left(\frac{G_{T+1}}{D}\right)\right]
\end{align*}
where we have used the fact that $\eta_t$ is non-increasing.

Next, we tackle $\rho^2\frac{8(1-\alpha_t)^2\eta_{t-1}^3}{\alpha_t^5G_t}$. We have
\begin{align*}
    \sum_{t=2}^{T+1}\E\left[ \frac{ \eta_{t-1}^3}{\alpha_t^5 G_t}\right]&\le \sum_{t=2}^{T+1}\E\left[ \frac{ \eta_{t-1}^4}{\eta_T\alpha_t^5 G_{t-1}}\right]\\
    &\le \sum_{t=2}^{T+1}\E\left[ \frac{ C^{14}}{t\eta_T}\right]\\
    &\le C^{14}\log(T+2)\E[\eta_T^{-1}]
\end{align*}

Now, finally we turn to bounding $-\left(\frac{1}{\alpha_{t-1}^2G_{t-1}\eta_{t-2}} - \frac{1}{\alpha_t^2G_t\eta_{t-1}}+\frac{1}{\alpha_tG_t\eta_{t-1}}\right)\|\epsilon_{t-1}\|^2$. To do this, we first upper-bound $\frac{1}{\alpha_t^2G_t\eta_{t-1}}-\frac{1}{\alpha_{t-1}^2G_{t-1}\eta_{t-2}}$. Note that:
\begin{align*}
    \frac{1}{\alpha_t^2G_t\eta_{t-1}}-\frac{1}{\alpha_{t-1}^2G_{t-1}\eta_{t-2}}&\le \frac{1}{\alpha_t^2G_t\eta_{t-1}}-\frac{1}{\alpha_{t-1}^2G_t\eta_{t-2}}
\end{align*}
So now we can upper bound $\frac{1}{\alpha_t^2\eta_{t-1}}-\frac{1}{\alpha_{t-1}^2\eta_{t-2}}$ and divide the bound by $G_t$.
\begin{align*}
    \frac{1}{\alpha_t^2\eta_{t-1}}-\frac{1}{\alpha_{t-1}^2\eta_{t-2}}&=t^2\eta_{t-1}^3 G_{t-1}^2 - (t-1)^2 \eta_{t-2}^3G_{t-2}^2\\
    &= C^3(t^{5/7}G_{t-1}^{8/7}-(t-1)^{5/7}G_{t-2}^{8/7})\\
    &\le C^3 t^{5/7}(G_{t-1}^{8/7} - G_{t-2}^{8/7})+C^3(t^{5/7}-(t-1)^{5/7})G_{t-1}^{8/7}
\end{align*}
Next, we analyze $G_{t-1}^{8/7} - G_{t-2}^{8/7}$. Recall our definition $\delta_t = G_t - G_{t-1}$, and we have $0\le \delta_t \le 2\G^2$ for all $t$. Then by convexity of the function $x\mapsto x^{8/7}$, we have
\begin{align*}
    G_{t-1}^{8/7} - G_{t-2}^{8/7}&\le \frac{8\delta_{t-1}}{7}G_{t-1}^{1/7}\le \frac{16\G^2}{7} G_{t-1}^{1/7}
\end{align*}
Therefore we have
\begin{align*}
    \frac{1}{\alpha_t^2\eta_{t-1}}-\frac{1}{\alpha_{t-1}^2\eta_{t-2}}&\le \frac{16C^3\G^2}{7} t^{5/7}G_{t-1}^{1/7}+C^3(t^{5/7}-(t-1)^{5/7})G_{t-1}^{8/7}\\
    &=\frac{16C^3\G^2t^{1/7}}{7G_{t-1}^{4/7}} t^{4/7}G_{t-1}^{5/7}+C^3(t^{5/7}-(t-1)^{5/7})G_{t-1}^{8/7}\\
\intertext{Now use $G_{t-1}\ge \G^2 t^{1/4}$,}
    &\le \frac{16C^3\G^{6/7}}{7} t^{4/7}G_{t-1}^{5/7}+C^3(t^{5/7}-(t-1)^{5/7})G_{t-1}^{8/7}\\
    &\le \frac{16C^3\G^{6/7}}{7} t^{4/7}G_{t-1}^{5/7}+\frac{5C^3}{7(t-1)^{2/7}}G_{t-1}^{8/7}\\
\intertext{Use $G_{t-1}\le D+3\G^2 (t-1)$,}
    &\le \frac{16C^3\G^{6/7}}{7} t^{4/7}G_{t-1}^{5/7}+\frac{5C^3(D+3\G^2(t-1))^{3/7}}{7(t-1)^{2/7}}G_{t-1}^{5/7}\\
    &\le \frac{21C^3\G^{6/7}}{7} t^{4/7}G_{t-1}^{5/7}+\frac{5C^3D^{3/7}}{7(t-1)^{2/7}}G_{t-1}^{5/7}
\intertext{Use the definition of $D$,}
    &\le \frac{21C^3\G^{6/7}}{7} t^{4/7}G_{t-1}^{5/7}+\frac{5C}{7(t-1)^{2/7}}G_{t-1}^{5/7}
\intertext{Use $C\ge 1/\G^{3/7}$,}
    &\le \frac{26C^3\G^{6/7}}{7} t^{4/7}G_{t-1}^{5/7}
\end{align*}
Now observe that
\begin{align*}
    \frac{26C^3\G^{6/7}}{7} t^{4/7}G_{t-1}^{5/7}&\le \frac{26C^2\G^{6/7}}{7\alpha_t \eta_{t-1}} 
\end{align*}
So putting all this together, we have 
\begin{align*}
    -\left(\frac{1}{\alpha_{t-1}^2G_{t-1}\eta_{t-2}} - \frac{1}{\alpha_t^2G_t\eta_{t-1}}+\frac{1}{\alpha_tG_t\eta_{t-1}}\right)\le -\left(\frac{1}{\alpha_tG_t\eta_{t-1}} - \frac{26C^2\G^{6/7}}{7\alpha_tG_t \eta_{t-1}} \right)
\end{align*}
Then since we set $C$ so that $\frac{26C^2\G^{6/7}}{7}=1/2$, we obtain:
\begin{align*}
    -\left(\frac{1}{\alpha_{t-1}^2G_{t-1}\eta_{t-2}} - \frac{1}{\alpha_t^2G_t\eta_{t-1}}+\frac{1}{\alpha_tG_t\eta_{t-1}}\right)\le -\frac{1}{2\alpha_tG_t\eta_{t-1}}
\end{align*}
Putting all this together, we have shown:
\begin{align*}
    \sum_{t=1}^{T} \E[K_{t+1}\|\hat \epsilon_{t+1}\|^2 - K_t\|\hat \epsilon_t\|^2]\le  \E\left[\frac{\log(T+2)}{\eta_T}+\frac{1}{\eta_T}\log\left(\frac{G_{T+1}}{D}\right)-\sum_{t=1}^{T}\frac{\|\hat\epsilon_t\|^2}{2\alpha_{t+1}G_{t+1}\eta_t}\right]
\end{align*}

Now, define the potential $\Phi_t = \frac{3F(\bw_{t+1})}{\eta_t} + K_{t+1}\|\hat \epsilon_{t+1}\|^2$.
Then, by Lemma \ref{thm:onestepnormalized}, we obtain:
\begin{align*}
    &\Phi_{t}-\Phi_{t-1}\le -\|\nabla F(\bw_t)\| +8 \|\hat\epsilon_t\|+\frac{3L\eta_t}{2}\\
    &\quad\quad +3F(\bw_t)\left(\frac{1}{\eta_t}-\frac{1}{\eta_{t-1}}\right) + K_{t+1}\|\hat \epsilon_{t+1}\|^2-K_t\|\hat\epsilon_t\|^2
\end{align*}
So summing over $t$ and taking expectations yields:
\begin{align*}
    \E[\Phi_{T} - \Phi_0]&\le \E\left[\sum_{t=1}^T\frac{3L\eta_t}{2}- \|\nabla F(\bw_t)\|+ \frac{3M}{\eta_T} + \sum_{t=1}^T8\|\hat \epsilon_t\| -\frac{\|\hat\epsilon_t\|^2}{2\alpha_{t+1}G_{t+1}\eta_t}\right.\\
    &\quad\quad\quad\left.+\frac{1}{\eta_T}\log\left(\frac{G_{T+1}}{D}\right)+\frac{\log(T+2)}{\eta_T}\right]
\end{align*}
Now, we examine the term $\sum_{t=1}^T8\|\hat \epsilon_t\| -\frac{\|\hat\epsilon_t\|^2}{2\alpha_{t+1}G_{t+1}\eta_t}$. By Cauchy-Schwarz we have:
\begin{align*}
    \sum_{t=1}^T8\|\hat \epsilon_t\| \le 8\sqrt{\sum_{t=1}^T\frac{\|\hat\epsilon_t\|^2}{2\alpha_{t+1}G_{t+1}\eta_t}\sum_{t=1}^T2\alpha_{t+1}G_{t+1}\eta_t}
\end{align*}
Therefore 
\begin{align*}
    &\sum_{t=1}^T8\|\hat \epsilon_t\| -\frac{\|\hat\epsilon_t\|^2}{2\alpha_{t+1}G_{t+1}\eta_t}\le\sup_M 8\sqrt{M\sum_{t=1}^T2\alpha_{t+1}G_{t+1}\eta_t}-M\\
    &\quad\quad\quad\le 32 \sum_{t=1}^T \alpha_{t+1} G_{t+1}\eta_t\\
    &\quad\quad\quad= 32\sum_{t=1}^T\frac{1}{(t+1)\eta_t}\\
    &\quad\quad\quad\le 32\sum_{t=1}^T\frac{1}{(t+1)\eta_T}\\
    &\quad\quad\quad\le \frac{32(\log(T+1))}{\eta_T}
\end{align*}
Finally, observe that since $G_t \ge \G^2t^{1/4}$, we have $\eta_t \le \frac{C}{\sqrt{T}}$. Therefore $\sum_{t=1}^T \eta_t \le 2C\sqrt{T}$. Putting all this together again, we have
\begin{align*}
    \sum_{t=1}^T \E[\|\nabla F(\bw_t)\|]&\le \Phi_0 + 3LC\sqrt{T} + \E[\eta_T^{-1}]\left[3M + \log(2\G^2(T+1)/D)+ \log(T+2) +32 (\log(T+1))\right]
\end{align*}
Observe that we have $\Phi_0 \le \frac{3M}{\eta_0} + K_1\G^2$.

Let us define $Z=3M + \log(2\G^2(T+1)/D) + \log(T+2) +32 (\log(T)+1) $. Then we have
\begin{align*}
    \sum_{t=1}^T \E[\|\nabla F(\bw_t)\|]&\le \Phi_1 + 3LC\sqrt{T} + \E[\eta_T^{-1}] Z
\end{align*}
Now we look carefully at the definition of $G_t$ and $\eta_t$. By Jensen inequality, we have
\begin{align*}
    \E[\eta_T^{-1}]&=\frac{1}{C}(T+1)^{3/7}\E\left[\left(D+2\G^2+\G T^{1/4}+\sum_{t=1}^T \|\nabla f(\bx_t, \xi_t)-\nabla f(\bx_t,\xi_t')\|^2\right)^{2/7}\right]\\
    &\le \frac{(T+1)^{3/7}(D+2\G^2+\G T^{1/4}+4T\sigma^2)^{2/7}}{C}\\
    &= O\left(\sqrt{T} + \sigma^{4/7} T^{5/7}\right)
\end{align*}
The Theorem statement now follows.
\end{proof}

\end{document}